\documentclass[envcountsame]{llncs}

\usepackage{xcolor}
\usepackage{hyperref}
\usepackage{xspace}
\usepackage{extarrows}

\usepackage{amsfonts}
\usepackage[T1]{fontenc}
\usepackage{tikz}
\usetikzlibrary{arrows,arrows.meta,automata,positioning} 

\newcommand{\lrot}[1]{\ell(#1)}
\newcommand{\sumpres}{source-bounded\xspace} 
\newcommand{\Wlog}{W.l.o.g.\xspace}
\newcommand{\mywlog}{w.l.o.g.\xspace}

\newcommand{\symg}[1]{S_{#1}}

\newcommand{\vect}[1]{\mathbf{#1}}
\newcommand{\vr}[1]{\vect{#1}}

\newcommand{\vas}{\text{\sc vas}\xspace}
\newcommand{\vass}{\text{\sc vass}\xspace}

\newcommand{\parvass}[1]{$#1$-\vass}
\newcommand{\parvas}[1]{$#1$-\vas}

\newcommand{\norm}[1]{\lVert#1\rVert}

\newcommand{\NP}{\text{\sc NP}\xspace}

\newcommand{\pspace}{\text{\sc PSpa\-ce}\xspace}

\newcommand{\trans}[1]{\stackrel{#1}{\longrightarrow}}

\newcommand{\size}[1]{|#1|}

\newcommand{\Z}{\mathbb{Z}}

\newcommand{\N}{\mathbb{N}}

\newcommand{\V}{\mathcal{V}}

\newcommand{\set}[1]{\left\{#1\right\}}

\newcommand{\setof}[2]{\left\{#1 \, \middle\vert \, #2\right\}}

\begin{document}

\title{Reachability in 3-VAS\thanks{Partially supported by the NCN grant 2024/55/B/ST6/01674.}}
\author{{\L}ukasz Kami{\'n}ski\orcidID{0009-0004-1641-9049} \\ \and 
S{\l}awomir Lasota\orcidID{0000-0001-8674-4470}}

\institute{University of Warsaw, Poland \\
\email{\{l.kaminski5, s.lasota\}@uw.edu.pl}
}

\authorrunning{{\L}. Kami{\'n}ski \and S. Lasota}

\maketitle

\begin{abstract}
We settle the exact complexity of the reachability problem in 
(stateless) vector addition systems (VAS) in fixed low dimension.
In dimensions 2--4 it has only been known to be sandwiched between NP and PSPACE. 
We prove PSPACE-hardness of the reachability problem for symmetric 
vector addition systems in dimension 3 (3-VAS), a restricted fragment of
general 3-VAS.
Combined with previously established PSPACE upper bounds, our result settles 
the complexity of the problem to be
PSPACE-complete in 3-VAS and 4-VAS, as well as in their symmetric 
fragments.
\end{abstract}


\section{Introduction}

Petri nets, or equivalently \emph{vector addition systems} (\vas), or equivalently
a stateful variant thereof, vector addition systems with states (\vass),
are a well-estab\-lished model of concurrency with numerous applications 
\cite{HopcroftPansiot}.
The central algorithmic question for this model is the \emph{reachability problem},
which asks whether a given target configuration can be reached from a given initial configuration by a sequence of valid execution steps.
The decidability of this problem was established by Mayr in 1981 \cite{Mayr81},
and was subsequently refined by Kosaraju \cite{Kosaraju82} and Lambert \cite{Lambert92}.
An exponential-space lower bound had already been shown by Lipton in 1976 \cite{Lipton76}.
For nearly four decades, these remained the only known complexity bounds,
making reachability one of the most prominent open problems in the verification of concurrent systems.
Only in the last few years has substantial progress been made, culminating in an Ackermannian upper bound \cite{LS19},
a breakthrough non-elementary lower bound \cite{CLLLM19,jacm}, and finally matching Ackermannian lower bounds obtained independently in \cite{CO21} and \cite{L21}.

Consequently, an active line of research concerns the complexity of reachability in fixed dimension.
Under binary encoding, reachability for \vass is \NP-complete in dimension~1 and 
\pspace-complete in dimension~2 \cite{Blondin15}.
For dimension~3, we have now \textsc{2-ExpSpace} upper bound \cite{3vass} but no lower bound
except the one inherited from dimension 2.
The complexity is even less understood in the stateless setting.
Under binary encoding again, reachability in dimension 1 is \NP-complete, and
\pspace-completeness of \vass in dimension 2 \cite{Blondin15}
implies \pspace-hardness of \vas in dimension 5, 
via the classical simulation of states using 3 additional dimensions \cite{HopcroftPansiot}.
\pspace upper bound has been recently established for \vas in dimension 4 \cite{chen2026}.
The exact complexity of reachability for \vas in dimensions~2, 3, and~4 therefore remains open, 
sandwiched between \NP\ and \pspace.

One possible approach to understanding the complexity of reachability is to study restricted classes of vector addition systems.
A particularly natural restriction is \emph{symmetric} \vass, where the set of transitions is closed 
under permutations of coordinates%
\footnote{
Motivation to study symmetric \vass comes from \emph{data \vass} \cite{Lasota16,LNORW07,RF11}, an
extension of plain \vass with data, where decidability of the reachability problem remains open.
}.
Quite surprisingly,
the reachability problem in 
symmetric \vass is \pspace-complete regardless of dimension $d \geq 2$ \cite{KL25},
and thus exhibits a dramatic complexity drop compared with general \vass.
\pspace upper bound for symmetric \vas of arbitrary dimension follows immediately,
but no nontrivial lower bound was known for this model.

\paragraph{Contribution}

Our main contribution is a proof of \pspace-hardness for the reachability problem
in \vas of dimension 3 (\parvas 3), even
in the restricted fragment of symmetric \vas.
\begin{theorem} \label{thm:main}
  The reachability problem for symmetric \parvas 3 is \pspace-hard.
\end{theorem}

Combining Theorem \ref{thm:main} with the upper bound of \cite[Thm.~1.2]{chen2026},
we obtain \pspace-completeness of the reachability problem
for both \vas and symmetric \vas in dimensions $3$ and $4$.

\begin{corollary} \label{cor:main}
  The reachability problem is \pspace-complete
  for symmetric \parvas 3, \parvas 3, symmetric \parvas 4, and \parvas 4.
\end{corollary}

The exact complexity of the reachability problem remains open
for \parvas 2 and for symmetric \parvas 2;
currently, it is only known to lie between \NP and \pspace.

\section{Preliminaries}

A $d$-dimensional \emph{vector addition system with states} ($d$-\vass)
is a pair $\V = (Q, T)$, where $Q$ is a finite set of states
and $T \subset Q \times \Z^d \times Q$ is a finite set of transitions.
The stateless fragment, 
where $\size{Q} = 1$, $\V$ is denoted as \parvas d.
Coordinates $1, \ldots, d$ are often called \emph{counters}.
A \emph{configuration} of $\V$ is a pair $(p, \vr v) \in Q \times \N^d$,
denoted as $p(\vr v)$.
For \vas, we omit the state and identify a configuration with an element of $\N^d$.
As \emph{norm} of a configuration $p(\vr v)$ we take the max-norm
of its underlying vector, i.e., $\norm{p(\vr v)} = \max\setof{\size{\vr v(i)}}{i \in \set{1, \ldots, d}}$.
For configurations $p(\vr v)$ and $q(\vr w)$ of $\V$, we have a step
$p(\vr v) \trans{} q(\vr w)$ if and only if $(p, \vr w - \vr v, q) \in T$.
Any sequence of configurations
$p_0(\vr v_0), p_1(\vr v_1), \ldots, p_n(\vr v_n)$
related by the step relation, namely satisfying
$p_{i-1}(\vr v_{i-1}) \trans{} p_i(\vr v_i)$ for $i = 1, \ldots, n$, we call
a \emph{run} from $p_0(\vr v_0)$ to $p_n(\vr v_n)$.
If such a run exists, we also say that $p_n(\vr v_n)$ is \emph{reachable} from $p_0(\vr v_0)$.

One of the most fundamental computational problems for \vass is
the \emph{reachability problem}:
\begin{quote}
given a \parvass d $\V$ together with two configurations, a source $s$ and a target $t$,
determine whether $\V$ has a run from $s$ to $t$.
\end{quote}

Symmetric \vass are a syntactic fragment of \vass where
the set of transitions is invariant under permutations of coordinates \cite{KL25}.
Let $S_d$ denote the symmetric group containing all permutations of $\set{1, \ldots, d}$.
The group $S_d$ acts on vectors $\vr w\in\Z^d$ by permuting coordinates:
a permutation $\sigma\in\symg d$ maps $\vr w$ to the vector $\sigma(\vr w)$ defined by
$\sigma(\vr w)(i) = \vr w(\sigma^{-1}(i))$, or equivalently $\sigma(\vr w)(\sigma(i)) = \vr w(i)$,
for $i = 1, \ldots, d$.
A \vass $\V = (Q, T)$ is symmetric
if 
$$T = \sigma(T) = \setof{(p, \sigma(\vr w), q)}{(p, \vr w, q) \in T}$$ 
for every $\sigma \in S_d$.
  Figure \ref{fig:example} shows a symmetric \parvass 2 with
  states $Q = \set{p, q}$ and the following transitions:
  \begin{align*}
    &(q, \ (-1, -1), \ p), \\
    &(p, \ (0,2),\ q), && (p, \ (2,0), \ q), \\
    &(p, \ (-1, 3), \ p), && (p, \ (3, -1), \ p), \\
    &(q, \ (1, -1), \ q), && (q, \ (-1, 1), \ q).
  \end{align*}
%
  The configuration $p(2,2)$ is reachable from $p(1,1)$
  via the run
  $$p(1,1) \trans{} p(0,4) \trans{} q(2,4) \trans{} q(3,3) \trans{} p(2,2)$$
  while
  $p(2,1)$ is not, since every transition preserves the parity of
  the sum of the counter values.

\begin{figure}
\begin{center}
\scalebox{0.7}
{\begin{tikzpicture} [draw=black!70,
  fill=blue!20,
  node distance = 4cm, 
  on grid, 
  auto,
  every initial by arrow/.style = {thin}]
\node (q0) [state,
scale=1.4,
fill=blue!20,
  initial text = {}] {$p$};
\node (q1) [state,
scale=1.4,
fill=blue!20,
  right = of q0] {$q$};
\path [-stealth, thick]
   (q0) edge[bend left] node[scale=1.5] {$\substack{(0,2) \\ (2, 0)}$}  (q1)
   (q1) edge[bend left] node[scale=1.1] {$(-1,-1)$}  (q0)
   (q0) edge [loop above] node[scale=1.5] {$\substack{(-1, 3) \\ (3, -1)}$}()
   (q1) edge [loop above] node[scale=1.5] {$\substack{(1, -1) \\ (-1, 1)}$}();
\end{tikzpicture}}
\end{center}
\caption{A symmetric \parvass 2.}
\label{fig:example}
\end{figure}

\begin{remark}
It is natural to assume that transition set $T$ of a symmetric \parvass 
d is represented succinctly, by 
listing one representative per every orbit of the action of $S_d$.
Therefore,
the reachability problem in symmetric \vass does not immediately reduce polynomially to the one in
general \vass due to exponential blow-up necessary to explicitly list all transitions.
However,
when dimension is fixed, as in this paper, the reduction is polynomial.
\end{remark}

A \parvass d $\V$ together with a configuration $s$
is called \emph{\sumpres} if norms of all configurations $s'$
reachable from $s$ satisfy $\norm{s'} \leq \norm{s}$.
In such case we also apply the term \emph{\sumpres} to the whole reachability instance
$(\V, s, t)$.

\section{The hardness proof}

As an intermediate step in proving Theorem \ref{thm:main},
we use the \pspace lower bound of \cite{KL25} for the reachability problem in symmetric \parvass 3.
For the sake of completeness, we recall the proof
in the form needed below, namely restricting to \sumpres instances.

\begin{lemma} \label{thm:3sumpres}
  The reachability problem for \sumpres symmetric \parvass 3 is \pspace-hard.
\end{lemma}

\begin{proof}
  We provide a reduction from the \pspace-complete reachability problem
  for \emph{bounded} \parvass 1 \cite{Fearnley_2015}. 
  An instance of this problem consists of
  a \parvass 1 $\V = (Q, T)$,
  configurations $s, t$, and a bound $B \in \N$.
  The problem asks whether $\V$ has a run from
  $s$ to $t$ in which the counter value is always at most $B$.
  Given $\V$, $s$, $t$, and $B$, we construct a symmetric \parvass 3
  $\V' = (Q', T')$ and
  two configurations $s'$ and $t'$ such that $\V$ has a $B$-bounded run from $s$ to $t$
  if and only if $\V'$ has a run from $s'$ to $t'$.
  Moreover, the instance $(\V', s', t')$ will be \sumpres.
  \Wlog we assume
  that the counter value in $s$ is $0$.

  The idea of the reduction is to use two counters of $\V'$ to store the counter value of $\V$.
  More precisely, a configuration $c = q(m)$ reachable from $s$ in $\V$ will be
  simulated by the configuration $\overline c = q(m, m, 2B-m)$ of $\V'$.
  Consequently, we put $s' = \overline s$ and $t' = \overline t$, and note
  that $\norm {\overline s} = 2B$.
  We let $Q' = Q \cup Q^+ \cup Q^-$, 
  where $Q^+ = \set{q^+ : q \in Q}$ and
  $Q^- = \set{q^- : q \in Q}$ are disjoint copies of $Q$.
  In the definition of the transitions of $\V'$
  we assume, \mywlog,
  that the effect of every transition of $\V$ is even,
  while the bound $B$ is odd.
  This guarantees that the counter of $\V$ never reaches exactly the value $B$ along any $B$-bounded
  run from $s$.
  We also assume, \mywlog, that all transitions of $\V$ have nonzero effect.
  For every transition $(p, n, q) \in T$, where $n > 0$, we add to $T'$ 
  six transitions, namely
  \[
    (p, (B+n, B+n, -B-n), q^+), \qquad \qquad 
    (q^+, (-B, -B, B), q),
  \]
  and their symmetric counterparts:
  \begin{align*}
    & (p, (B+n, -B-n, B+n), q^+), &&
    (q^+, (-B, B, -B), q), \\
    & (p, (-B-n, B+n, B+n), q^+), &&
    (q^+, (B, -B, -B), q).
  \end{align*}
  We observe that a step $p(m) \trans{} q(m+n)$ in $\V$ that
  does not exceed the bound $B$ can be simulated from
  the configuration $\overline{p(m)} = p(m, m, 2B-m)$ of $\V'$
  using the first two transitions above, namely
  \begin{align*}
    p(m, m, 2B-m) &\trans{} q^+(B+m+n, B+m+n, B-m-n) \\
    &\trans{} q(m+n, m+n, 2B-m-n) = \overline{q(m+n)}.
  \end{align*}
  On the other hand, the symmetric counterparts of these transitions
  cannot be used: replacing the first
  transition by any of its
  symmetric counterparts would decrease the first or second counter below zero
  (since, by our initial assumption, $m<B$), while
  replacing the second transition by any of its symmetric counterparts would decrease
  the third counter below zero (relying on the strict inequality $m+n<B$).
  Furthermore, the two-transition sequence in $\V'$ shown above cannot cause the simulated
  counter of $\V$ to exceed the bound $B$:
  the first transition cannot be executed if $m+n > B$,
  since it would decrease the last counter below zero. 
  
  Analogously, for every transition $(p, -n, q) \in T$, where $n > 0$, 
  we add to $T'$ six transitions, namely
  \[
    (p, (B, B, -B), p^-), \qquad \qquad 
    (p^-, (-B-n, -B-n, B+n), q), 
  \]
  and their symmetric counterparts:
  \begin{align*}
    & (p, (-B, B, B), p^-), && 
    (p^-, (B+n, -B-n, -B-n), q), \\
    & (p, (B, -B, B), p^-), && 
    (p^-, (-B-n, B+n, -B-n), q).
  \end{align*}
  Again, we observe that a step $p(m) \trans{} q(m-n)$ in $\V$
  can be simulated from
  the configuration $\overline{p(m)} = p(m, m, 2B-m)$ of $\V'$
  using the first two transitions above, namely
  \begin{align*}
    p(m, m, 2B-m) &\trans{} p^-(B+m, B+m, B-m) \\
    &\trans{} q(m-n, m-n, 2B-m+n) = \overline{q(m-n)}.
  \end{align*}
  Moreover, the symmetric counterparts of these transitions cannot be used, for the same reason as above:
  replacing the first transition by any of its
  symmetric counterparts would decrease the first or second counter below zero, and
  replacing the second transition by any of its symmetric counterparts would decrease
  the last counter below zero.
  For the same reason, the two-transition sequence in $\V'$ shown above cannot
  make the simulated counter negative, i.e., it cannot yield
  $m-n<0$. 
 
  
  The above observations establish the correctness of the reduction: there is a $B$-bounded run from
  $s$ to $t$ in $\V$ if and only if there is a run from
  $\overline s$ to $\overline t$ in $\V'$.
  Moreover, the instance $(\V', s', t')$ is \emph{\sumpres},
  as every configuration $c'$ reachable in $\V'$ from $s'$ has the norm at most $2B=\norm{s'}$.
  \qed
\end{proof}


Having Lemma \ref{thm:3sumpres}, we are prepared to prove Theorem \ref{thm:main}.

\begin{proof}[of Theorem \ref{thm:main}]
  We provide a reduction from the reachability problem in \sumpres symmetric
  \parvass 3.
  The overall idea of the proof is inspired by the classical control-state 
  elimination construction of \cite{HopcroftPansiot}.

  Consider a \sumpres symmetric \parvass 3 $\V = (Q, T)$ together with 
  source and target configurations 
  $s, t$.
  Let $B := 2\norm s + 1$ to guarantee that
  the norm of each configuration $s'$
  reachable from $s$ satisfies $\norm{s'} < B / 2$.
  We can also assume \mywlog 
  that transitions change counter values by numbers smaller than $B / 2$,
  i.e., for every $(p, \vr w, q) \in T$ we have $\norm{\vr w} < B / 2$.
  
  As a preparatory step, we transform $\V$ to ensure
  that the length of every run from $s$ to $t$
  in $\V$ is a multiple of $3$, and that the only transitions with nonzero effect are
  these at positions congruent to $1$ modulo $3$ in such runs. 
  This can be achieved by adding two dummy copies of control states
  $Q_1 = \setof{q_1}{q \in Q}$ and $Q_2 = \setof{q_2}{q \in Q}$, and  
  replacing every transition $(p, \vr w, q)$ of $\V$ with the following three 
  transitions:
  \begin{align} \label{eq:3trans}
    (p, \vr w, q_1), \qquad \qquad (q_1, \vr 0, q_2), \qquad \qquad (q_2, \vr 0, q),
  \end{align}
  thus obtaining the new set of transitions $T$.
  The transformation preserves the symmetry and \sumpres{}ness of $\V$.
  
  Furthermore we assume, \mywlog, that $Q\cup Q_1\cup Q_2 = \set{1, 2, \ldots, N-1}$, 
  for some $N\in\N$.
  
  \medskip

  Having made these assumptions on $\V = (Q\cup Q_1\cup Q_2,T)$,
  we define a symmetric \parvas 3 $\V'$ with the transition set $T'$ that simulates $\V$, together with source
  and target configurations $s'$ and $t'$.
  For every control state $p \in Q \cup Q_1 \cup Q_2$, let $A_p := 2N(N-p)$, 
  and define the vector
  \[
    \vr v_{p} = (B\cdot A_p, \ B \cdot p, \ 0) \in \Z^3,
  \]
  to be used in the simulation of $\V$. The idea is to simulate a configuration
  $p(\vr w)$ of $\V$ by the configuration
  $\overline{q(\vr w)} = \vr v_p + \vr w \in \N^3$ of $\V'$
  (note that the notation $\overline{q(\vr w)}$ has now a different 
  meaning than in the proof of Lemma \eqref{thm:3sumpres}).
  Consequently, we put $s' := \overline s$ and $t' := \overline t$.
  We note three important intuitive ideas underlying the proof:
  \begin{enumerate}
    \item[(i)] The first coordinate of $\vr v_p$ is `very large', the second one is `large', 
    and the last one 
    is 0. This difference of orders of magnitude will be crucial for distinguishing the
    coordinates, what guarantees correctness of the simulation.
    
    \smallskip
    \item[(ii)] With increasing $p$, the second coordinate of $\vr v_p$ increases while the first one 
    decreases. This will ensure the right choice of $p$ in the simulation.

    \smallskip
    \item[(iii)] Since $(\V, s)$ is \sumpres, all entries of $\vr w$
  are smaller than $B$, whereas all entries of
  $\vr v_{p}$ are multiples of $B$.
  This separation of scales will allow us, intuitively speaking,
  to distinguish $\vr w$ from $\vr v_{p}$ in the simulation of $\V$ by $\V'$.
  \end{enumerate}
  When defining $T'$ we will use
  the two left rotations of $\vr v_p$:
  \[
    \lrot{\vr v_p} \ := \ (B \cdot p, \ 0, \ B\cdot A_p), \qquad 
    \lrot{\lrot{\vr v_p}} \ := \ (0, \ B\cdot A_p, \ B \cdot p).
  \]
  For any two distinct control states $p,q\in Q\cup Q_1 \cup Q_2$ we define the vector
  \begin{align} \label{eq:rotation}
    \vr v_{pq} \ := \ \lrot{\vr v_q} - \vr v_p \ = \ 
    (B \cdot q - B\cdot A_p, \ - B \cdot p, \ B \cdot A_q) \ \in \ \Z^3,
  \end{align}
  together with its two left-rotated versions:
  \begin{align} \label{eq:rotations}
    \lrot{\vr v_{pq}} \ := \ \lrot{\lrot{\vr v_q}} - \lrot{\vr v_p}, \qquad 
    \lrot{\lrot{\vr v_{pq}}} \ := \ \vr v_q - \lrot{\lrot{\vr v_p}}.
  \end{align}
It will be important later that the first coordinate of $\vr v_{pq}$ is at most
$-N\cdot B$, and its second coordinate is larger than $-N\cdot B$ but at most $-B$:
\begin{align} \label{eq:ineq}
\vr v_{pq}(1) \ \leq \ -N\cdot B \ < \ \vr v_{pq}(2) \ \leq \ -B.
\end{align}

\smallskip

  We define a symmetric \parvas 3 $\V'$ as follows: for every transition
  $(p, \vr w, q) \in T$, where $p,q\in Q\cup Q_1\cup Q_2$,
  we add to $T'$ the six transitions
  \begin{align} \label{eq:trans}
    \setof{\sigma(\vr v_{pq})+\vr w}{\sigma \in \symg 3}.
  \end{align}
  As $\V$ is symmetric, $\V'$ is symmetric as well.
  Note that $\lrot{\vr v_{pq}}$ and $\lrot{\lrot{\vr v_{pq}}}$ belong to
  $\setof{\sigma(\vr v_{pq})}{\sigma \in \symg 3}$.
  The correctness proof of the reduction relies on the observation that
  only certain transitions from \eqref{eq:trans} are usable by $\V'$ when it starts
  from the source $\overline s$, namely these listed in \eqref{eq:3trans_sim} below.
  Thus, the symmetry of $\V'$ is not fully exploitable.

  \medskip

  We need to argue that there is a run of $\V$ from $s$ to $t$
  if and only if there is a run of $\V'$ from $\overline s$ to $\overline t$.

  \medskip

  For the `only-if' implication, we observe that the three consecutive steps of $\V$
  induced by firing the three transitions from \eqref{eq:3trans}, namely
  \begin{align} \label{eq:3trans_V}
    p(\vr u) \trans{} q_1(\vr u + \vr w) \trans{} q_2(\vr u + \vr w) \trans{}
    q(\vr u + \vr w),
  \end{align}
  are simulated from the configuration $\overline{p(\vr u)} = \vr v_p + \vr u$ of $\V'$ 
  by three consecutive steps using the three corresponding transitions in $T'$:
  \begin{align} \label{eq:3trans_sim}
    \vr v_{pq_1} + \vr w, \qquad 
    \lrot{\vr v_{q_1 q_2}} + \vr 0, \qquad 
    \lrot{\lrot{\vr v_{q_2 q}}} + \vr 0.
  \end{align}
  The three steps of $\V'$ reach $\vr v_q + \vr u + \vr w = \overline{q(\vr u + \vr w)}$,
  by the equalities \eqref{eq:rotation} and \eqref{eq:rotations}.
  Therefore, by a straightforward induction on the length of the run,
  every run of $\V$ from $s$ to $t$ is simulated by a run of $\V'$ from
  $\overline s$ to $\overline t$.
  
  \medskip

  For the converse implication, consider a configuration $p(\vr u)$ reachable
  from $s$ in $\V$, and the corresponding configuration
  $\overline{p(\vr u)} = \vr v_p + \vr u$ of $\V'$.
  We will argue that the three steps of $\V'$ as shown in \eqref{eq:3trans_sim},
  corresponding to some three steps of $\V$ as in \eqref{eq:3trans_V},
  are the only possible 
  steps from $\overline{p(\vr u)}$, and 
  if any of these steps is not fireable, $\V'$ reaches a deadlock configuration.
  As the first step, we 
  use the following facts: 
  \begin{align} \label{eq:ass}
    \vr v_p = (B\cdot A_p, \ B \cdot p, \ 0), \qquad 
    A_p\geq N, \qquad
    p < N, \qquad \norm{\vr u} < B/2
  \end{align}
  (the latter inequality holds since $(\V, s)$ is \sumpres),
  to deduce:
  \begin{claim}
  The only possible step of $\V'$ from $\vr v_p + \vr u$ is to use a transition
  of the form $\vr v_{pq_1} + \vr w$, for some transition $(p, \vr w, q_1) \in T$.
  \end{claim}
  \paragraph{Proof of the claim.}
  First, using any other transition of the form $\sigma(\vr v_{pq_1}) + \vr w$, with
  $\sigma$ different from the identity,
  would decrease the second or third counter below zero, due to
  the inequalities \eqref{eq:ineq}.
  We rely here on the intuitive point (i) above. Indeed, 
  by the inequalities in \eqref{eq:ass} we observe that  
  the first coordinate is the only `very large' one, namely at least as large as 
  $B\cdot N$, and the second coordinate is not `very large'
  but it is still `large', namely at least as large as $B$.
  We rely here also on the point (iii), indeed, the third coordinate
  is smaller than $B/2$
  (cf.~the inequalities \eqref{eq:ineq}).
  Second, using any other transition of the form $\sigma(\vr v_{qq_1}) + \vr w$ where
  $q\neq p$,
  would decrease the first or the second counter below zero
  if $\sigma$ is the identity (we rely here on the point (ii) above);
  and would decrease the second or the third counter below zero
  if $\sigma$ is different from the identity -- here we argue the same way as above,
  again, we rely here on the point (i).
  The claim is thus proved. \qed

  \medskip

  Firing a transition of the form $\vr v_{pq_1} + \vr w$ leads to the configuration
  $\lrot{\vr v_{q_1}} + \vr u + \vr w$ of $\V'$.
  However, it could happen that
  $\vr u + \vr w$ has negative components that are compensated by
  some `very large' or `large' positive component of $\lrot{\vr v_{q_1}}$,
  so that the step of $\V'$ does not correspond to a step of $\V$.
  %
  %
  To rule out this possibility, we argue as follows.
  Using the following facts:
  \[
    \lrot{\vr v_{q_1}} = (B \cdot {q_1}, \ 0, \ B\cdot A_{q_1}), \qquad 
    A_{q_1}\geq N, \qquad
    q_1 < N, \qquad \norm{\vr u + \vr w} < B
  \]
  (some components of $\vr u + \vr w$ might exceed $B/2$ when another component
  is negative, but all components are surely smaller than $B$),
  we deduce that the second coordinate of $\vr u + \vr w$ cannot be negative
  (since $\lrot{\vr v_{q_1}}$ has zero in this coordinate).
  Furthermore, by the same reasoning as in the claim above, we deduce that
  the only possible step of $\V'$ from $\lrot{\vr v_{q_1}} + \vr u + \vr w$ is
  to use a transition of the form $\lrot{\vr v_{q_1 q_2}} + \vr 0$.
  If fireable, this transition
  leads to the configuration $\lrot{\lrot{\vr v_{q_2}}} + \vr u + \vr w$ of $\V'$.
  Then once again, using the following facts:
  \[
    \lrot{\lrot{\vr v_{q''}}} = (0, \ B\cdot A_{q_2}, \ B \cdot q_2), \qquad 
    A_{q_2}\geq N, \qquad
    q_2 < N, \qquad \norm{\vr u + \vr w} < B,
  \]
  we deduce that the first coordinate of $\vr u + \vr w$ cannot be negative
  (since $\lrot{\lrot{\vr v_{q_2}}}$ has zero in this coordinate),
  and moreover
  the only possible step of $\V'$ from $\lrot{\lrot{\vr v_{q_2}}} + \vr u + \vr w$ is
  to use a transition of the form $\lrot{\lrot{\vr v_{q_2 q}}} + \vr 0$,
  as in the claim above.
  If fireable, this transition
  leads to the configuration
  $\vr v_q + \vr u + \vr w = \overline{q(\vr u + \vr w)}$ of $\V'$,
  which implies, again by the same reasoning,
  that the third coordinate of $\vr u + \vr w$ cannot be negative.
  Summing up, firing the three transitions checks nonnegativeness of 
  $\vr u + \vr w$ on all coordinates.

  \smallskip

  We conclude that the three steps of $\V'$ shown above are the only possible steps from
  $\overline{p(\vr u)}$, and if they are fireable, they lead to $\overline{q(\vr u + \vr w)}$.
  On the other hand, if any of these steps is not fireable, $\V'$ reaches a deadlock configuration,
  and consequently the run cannot reach the target $\overline t$.
  Therefore, using a straightforward induction on the length of a run we prove that
  every run of $\V'$ from $\overline s$ to $\overline t$ corresponds
  to (implies) a run of $\V$ from $s$ to $t$.
  This completes the proof of Theorem \ref{thm:main}.
  \qed
\end{proof}


\begin{remark}
  While the reachability problem in
  \parvas 3 straightforwardly reduces to
  the reachability problem in \parvas d for any $d>3$, the reduction
  does not work for symmetric \vas.
  Nevertheless, 
  Lemma \ref{thm:3sumpres} holds for any dimension $d \geq 2$ \cite{KL25}, even under
  restriction to   \sumpres instances, and in consequence
  the proof of Theorem \ref{thm:main} can be easily adapted to any dimension $d \geq 3$.
\end{remark}

\section{Final remarks}

We have shown that the reachability problem
for symmetric \parvas 3 is \pspace-hard.
As a consequence, reachability for general
\parvas 3 is also \pspace-hard.
Combining our result with upper bounds
from \cite{chen2026,KL25}, we obtain Corollary \ref{cor:main}:
\pspace-completeness of the reachability problem
for \vas and symmetric \vas in dimensions $3$ and $4$.

One can observe that the reduction
used in the proof of Theorem \ref{thm:main}
does not become substantially harder due to the assumed symmetry of \parvas 3.
This suggests that symmetric \vass may serve
as a useful intermediate model for studying the complexity of decision problems 
in low-dimensional \vas.

The techniques developed in this paper do not seem
to extend to two-dimensional \vas. 
At present, the complexity of the reachability problem for this model 
remains open, sandwiched between \NP and \pspace.
Indeed, we have the following line of reductions between the models:%
\footnote{The second reduction has been recently noticed by Karol Węgrzycki and Anubhav Dhar.}
\newcommand{\reducesto}{\ \ \rightarrow\ \ }
\[
\text{\parvas 1} \reducesto
\text{bounded \parvas 1} \reducesto
\text{symmetric \parvas 2} \reducesto
\text{\parvas 2} \reducesto
\text{\parvass 2},
\]
with the first one being \NP-complete, the last one being \pspace-complete,
and the complexity of all the intermediate cases unknown.
(As in the proof of Lemma \ref{thm:3sumpres}, the bounded version of the reachability problem asks for a run that never
exceeds the value given as part of the input, encoded in binary.)

\begin{question}
  What is the complexity of the reachability problem in 
  bounded \parvas 1, symmetric \parvas 2,
  and general \parvas 2?
\end{question}

Notably, all the complexities are known for the analogous line of  stateful models,
namely the first one is \NP-complete and all others are \pspace-complete:
\[
\text{\parvass 1} \reducesto
\text{bounded \parvass 1} \reducesto
\text{symmetric \parvass 2} \reducesto
\text{\parvass 2}.
\]

\begin{credits}
\subsubsection*{\ackname}
We are grateful to Henry Sinclair-Banks for initiation of investigation of
low-dimensional VAS, and for many valuable discussions.
\end{credits}

\bibliography{bib,pn-bib}

\appendix

\end{document}